%% file: IndKernelLearningA.tex
\documentclass[11pt]{article}
\usepackage{array,psfrag,graphics,epsfig,fullpage,multirow,multicol,url,color,algorithmic,algorithm}
\usepackage[agsmcite]{harvard}

\citationmode{abbr}

\definecolor{webblue}{rgb}{0,0,1}
\definecolor{webgreen}{rgb}{0,.5,0}
\definecolor{webbrown}{rgb}{.6,0,0}

\input defs.tex

\title{Support Vector Machine Classification with Indefinite Kernels}

\author{Ronny Luss\thanks{ORFE Department, Princeton University,
Princeton, NJ 08544. \texttt{rluss@alumni.princeton.edu}} \and Alexandre
d'Aspremont\thanks{ORFE Department, Princeton University, Princeton,
NJ 08544. \texttt{aspremon@princeton.edu}}}

\begin{document}
\maketitle

\begin{abstract}
We propose a method for support vector machine classification using indefinite kernels. Instead of directly minimizing or stabilizing a nonconvex loss function, our algorithm simultaneously computes support vectors and a proxy kernel matrix used in forming the loss. This can be interpreted as a penalized kernel learning problem where indefinite kernel matrices are treated as noisy observations of a true Mercer kernel. Our formulation keeps the problem convex and relatively large problems can be solved efficiently using the projected gradient or analytic center cutting plane methods. We compare the performance of our technique with other methods on several standard data sets.
\end{abstract}


\section{Introduction}
Support vector machines (SVM) have become a central tool for solving binary classification problems. A critical step in support vector machine classification is choosing a suitable kernel matrix, which measures similarity between data points and must be positive semidefinite because it is formed as the Gram matrix of data points in a reproducing kernel Hilbert space. This positive semidefinite condition on kernel matrices is also known as Mercer's condition in the machine learning literature.  The classification problem then becomes a linearly constrained quadratic program. Here, we present an algorithm for SVM classification using indefinite kernels\footnote{A preliminary version of this paper appeared in the proceedings of the Neural Information Processing Systems (NIPS) 2007 conference and is available at \texttt{http://books.nips.cc/nips20.html}}, i.e. kernel matrices formed using similarity measures which are not positive semidefinite.

Our interest in indefinite kernels is motivated by several observations. First, certain similarity measures take advantage of application-specific structure in the data and often display excellent empirical classification performance. Unlike popular kernels used in support vector machine classification, these similarity matrices are often indefinite, so do not necessarily correspond to a reproducing kernel Hilbert space. (See \citeasnoun{Ong2004} for a discussion.)

In particular, an application of classification with indefinite kernels to image classification using Earth Mover's Distance was discussed in \citeasnoun{Zamo2007}. Similarity measures for protein sequences such as the Smith-Waterman and BLAST scores are indefinite yet have provided hints for constructing useful positive semidefinite kernels such as those decribed in \citeasnoun{Saig2004} or have been transformed into positive semidefinite kernels with good empirical performance (see \citeasnoun{Lanc2003}, for example).  Tangent distance similarity measures, as described in \citeasnoun{Sima1998} or \citeasnoun{Haas2002}, are invariant to various simple image transformations and have also shown excellent performance in optical character recognition. Finally, it is sometimes impossible to prove that some kernels satisfy Mercer's condition or the numerical complexity of evaluating the exact positive kernel is too high and a proxy (and not necessarily positive semidefinite) kernel has to be used instead (see \citeasnoun{Cutu07}, for example). In both cases, our method allows us to bypass these limitations. Our objective here is to derive efficient algorithms to directly use these indefinite similarity measures for classification.

Our work closely follows, in spirit, recent results on kernel learning (see \citeasnoun{Lanc2004} or \citeasnoun{Ong2005}), where the kernel matrix is learned as a linear combination of given kernels, and the result is explicitly constrained to be positive semidefinite. While this problem is numerically challenging, \citeasnoun{Bach2004}  adapted the SMO algorithm to solve the case where the kernel is written as a positively weighted combination of other kernels. In our setting here, we never \emph{numerically} optimize the kernel matrix because this part of the problem can be solved explicitly, which means that the complexity of our method is substantially lower than that of classical kernel learning algorithms and closer in practice to the algorithm used in \citeasnoun{Sonn2006}, who formulate the multiple kernel learning problem of \citeasnoun{Bach2004} as a semi-infinite linear program and solve it with a column generation technique similar to the analytic center cutting plane method we use here.

\subsection{Current results}
Several methods have been proposed for dealing with indefinite kernels in SVMs.  A first direction embeds data in a pseudo-Euclidean (pE) space: \citeasnoun{Haas2005}, for example, formulates the classification problem with an indefinite kernel as that of minimizing the distance between convex hulls formed from the two categories of data embedded in the pE space.  The nonseparable case is handled in the same manner using reduced convex hulls. (See \citeasnoun{Benn2000} for a discussion on geometric interpretations in SVM.)

Another direction applies direct spectral transformations to indefinite kernels: flipping the negative eigenvalues or shifting the eigenvalues and reconstructing the kernel with the original eigenvectors in order to produce a positive semidefinite kernel (see \citeasnoun{Wu2005} and \citeasnoun{Zamo2007}, for example). Yet another option is to reformulate either the maximum margin problem or its dual in order to use the indefinite kernel in a convex optimization problem. One reformulation suggested in \citeasnoun{Lin2003} replaces the indefinite kernel by the identity matrix and maintains separation using linear constraints. This method achieves good performance, but the convexification procedure is hard to interpret. Directly solving the nonconvex problem sometimes gives good results as well (see \citeasnoun{Wozn2006} and \citeasnoun{Haas2005}) but offers no guarantees on performance.

\subsection{Contributions}
In this work, instead of directly transforming the indefinite kernel, we simultaneously learn the support vector weights and a proxy Mercer kernel matrix by penalizing the distance between this proxy kernel and the original, indefinite one. Our main result is that the kernel learning part of that problem can be solved explicitly, meaning that the classification problem with indefinite kernels can simply be formulated as a perturbation of the positive semidefinite case.

Our formulation can be interpreted as a penalized kernel learning problem with uncertainty on the input kernel matrix. In that sense, indefinite similarity matrices are seen as noisy observations of a true positive semidefinite kernel and we learn a kernel that increases the generalization performance.  From a complexity standpoint, while the original SVM classification problem with indefinite kernel is nonconvex, the penalization we detail here results in a convex problem, and hence can be solved efficiently with guaranteed complexity bounds.

The paper is organized as follows.  In Section \ref{s:max-min} we formulate our main classification result and detail its interpretation as a penalized kernel learning problem. In Section \ref{s:algos} we describe three algorithms for solving this problem.  Section \ref{s:extensions} discusses several extensions of our main results.  Finally, in Section \ref{s:num-res}, we test the numerical performance of these methods on various data sets.

\noindent \\
\textbf{Notation}\\
\noindent
We write $\symm^n$ ($\symm^n_+$) to denote the set of symmetric (positive-semidefinite) matrices of size $n$.  The vector $e$ is the $n$-vector of ones. Given a matrix $X$, $\lambda_i\left(X\right)$ denotes the $i^{th}$ eigenvalue of $X$. $X_+$ is the positive part of the matrix $X$, i.e. $X_+=\sum_i\mbox{max}(0,\lambda_i)v_iv_i^T$ where $\lambda_i$ and $v_i$ are the $i^{th}$ eigenvalue and eigenvector of the matrix $X$.  Given a vector $x$, $\|x\|_1=\sum{|x_i|}$.

\section{SVM with indefinite kernels} \label{s:max-min}
In this section, we modify the SVM kernel learning problem and formulate a penalized kernel learning problem on indefinite kernels. We also detail how our framework applies to kernels that satisfy Mercer's condition.

\subsection{Kernel learning}
Let $K\in\symm^n$ be a given kernel matrix and let $y\in\reals^n$ be the vector of labels, with $Y=\diag(y)$, the matrix with diagonal $y$.  We formulate the kernel learning problem as in \citeasnoun{Lanc2004}, where the authors minimize an upper bound on the misclassification probability when using SVM with a given kernel $K$.  This upper bound is the generalized performance measure
\BEQ \label{eq:gen_perf_meas}
\omega_C(K)=\max_{\{0\leq\alpha\leq C,\alpha^Ty=0\}}\alpha^Te-\Tr(K(Y\alpha)(Y\alpha)^T)/2
\EEQ
where $\alpha\in\reals^n$ and $C$ is the SVM misclassification penalty. This is also the classic 1-norm soft margin SVM problem.  They show that $\omega_C(K)$ is convex in $K$ and solve problems of the form
\BEQ \label{eq:kern_learn_lanc}
\min_{K\in\mathcal{K}} \omega_C(K)
\EEQ
in order to learn an optimal kernel $K^*$ that achieves good generalization performance. When $\mathcal{K}$ is restricted to convex subsets of $S^n_+$ with constant trace, they show that problem (\ref{eq:kern_learn_lanc}) can be reformulated as a convex program. Further restrictions to $\mathcal{K}$ reduce (\ref{eq:kern_learn_lanc}) to more tractable optimization problems such as semidefinite and quadratically constrained quadratic programs.  Our goal is to solve a problem similar to (\ref{eq:kern_learn_lanc}) by restricting the distance between a proxy kernel used in classification and the original indefinite similarity measure.

\subsection{Learning from indefinite kernels}
\label{ss:indKernelLearning} The performance measure in (\ref{eq:gen_perf_meas}) is the dual of the SVM classification problem with hinge loss and quadratic penalty.  When $K$ is positive semidefinite, this problem is a convex quadratic program. Suppose now that we are given an indefinite kernel matrix $K_0\in\symm^n$.  We formulate a new instance of problem (\ref{eq:kern_learn_lanc}) by restricting $K$ to be a positive semidefinite kernel matrix in some given neighborhood of the original (indefinite) kernel matrix $K_0$ and solve
\[
\DS \min_{\{K\succeq0,~\|K-K_0\|^2_F\leq\beta\}}~\max_{\{\alpha^Ty=0,~0\leq\alpha\leq C\}}~\alpha^Te-\frac{1}{2}\Tr(K(Y\alpha)(Y\alpha)^T)
\]
in the variables $K\in\symm^n$ and $\alpha\in\reals^n$, where the parameter $\beta >0$ controls the distance between the original matrix $K_0$ and the proxy kernel $K$.  This is the kernel learning problem (\ref{eq:kern_learn_lanc}) with $\mathcal{K}=\{K\succeq0,~\|K-K_0\|^2_F\leq\beta\}$.  The above problem is infeasible for small values of $\beta$, so we replace here the hard constraint on $K$ by a penalty $\rho$ on the distance between the proxy kernel and the original indefinite similarity matrix and solve instead
\BEQ \label{eq:kern_learn_hard}
\DS \min_{\{K\succeq0\}}~\max_{\{\alpha^Ty=0,~0\leq\alpha\leq C\}}~\alpha^Te-\frac{1}{2}\Tr(K(Y\alpha)(Y\alpha)^T)+\rho\|K-K_0\|^2_F
\EEQ
Because~(\ref{eq:kern_learn_hard}) is convex-concave and the inner maximization has a compact feasible set, we can switch the max and min to form the dual
\BEQ \label{eq:kern_learn}
\max_{\{\alpha^Ty=0,0\leq\alpha\leq C\}} ~ \min_{\{K\succeq0\}} ~ \alpha^Te-
\frac{1}{2}\Tr(K(Y\alpha)(Y\alpha)^T)+\rho\|K-K_0\|^2_F
\EEQ
in the variables $K\in\symm^n$ and $\alpha\in\reals^n$.

We first note that problem (\ref{eq:kern_learn}) is a convex optimization problem.  The inner minimization problem is a convex conic program on $K$. Also, as the pointwise minimum of a family of concave quadratic functions of $\alpha$, the solution to the inner problem is a concave function of $\alpha$, hence the outer optimization problem is also convex (see \citeasnoun{Boyd2004} for further details). Thus, (\ref{eq:kern_learn}) is a concave maximization problem subject to linear constraints and is therefore a convex problem in $\alpha$. Our key result here is that the inner kernel learning optimization problem in (\ref{eq:kern_learn}) can be solved in closed form.

\begin{theorem} \label{th:kstar}
Given a similarity matrix $K_0\in S^n$, a vector $\alpha\in\reals^n$ of support vector coefficients and the label matrix $Y=\diag(y)$, the optimal kernel in problem (\ref{eq:kern_learn}) can be computed explicitly as:
\BEQ\label{eq:kstar}
K^*=(K_0+(Y\alpha)(Y\alpha)^T/(4\rho))_+
\EEQ
where $\rho\geq 0$ controls the penalty.
\end{theorem}
\begin{proof}
For a fixed $\alpha$, the inner minimization problem can be written out as
\[
\min_{\{K\succeq0\}} ~
\alpha^Te+\rho(\Tr(K^TK)-2\Tr(K^T(K_0+\frac{1}{4\rho}(Y\alpha)(Y\alpha)^T))+\Tr(K_0^TK_0))
\]
where we have replaced $\|K-K_0\|^2_F=\Tr((K-K_0)^T(K-K_0))$ and collected similar terms.  Adding and subtracting the constant $\rho\Tr((K_0+\frac{1}{4\rho}(Y\alpha)(Y\alpha)^T)^T(K_0+\frac{1}{4\rho}(Y\alpha)(Y\alpha)^T))$ shows that the inner minimization problem is equivalent to the problem
\[
\BA {ll}
\mbox{minimize} & \|K-(K_0+\frac{1}{4\rho}(Y\alpha)(Y\alpha)^T)\|^2_F\\
\mbox{subject to} & K\succeq0
\EA \]
in the variable $K\in\symm^n$, where we have dropped the remaining constants from the objective. This is the projection of the matrix  $K_0+(Y\alpha)(Y\alpha)^T/(4\rho)$ on the cone of positive semidefinite matrices, which yields the desired result.
\end{proof}

Plugging the explicit solution for the proxy kernel derived in (\ref{eq:kstar}) into the classification problem~(\ref{eq:kern_learn}), we get
\BEQ \label{eq:kern_learn2}
\DS \max_{\{\alpha^Ty=0,~0\leq\alpha\leq C\}} ~ \alpha^Te-
\frac{1}{2}\Tr(K^*(Y\alpha)(Y\alpha)^T)+\rho\|K^*-K_0\|^2_F
\EEQ
in the variable $\alpha\in\reals^n$, where $(Y\alpha)(Y\alpha)^T$ is the rank one matrix with coefficients $y_i\alpha_i\alpha_jy_j$.  Problem (\ref{eq:kern_learn2}) can be cast as an eigenvalue optimization problem in the variable $\alpha$.  Letting the eigenvalue decomposition of $K_0+(Y\alpha)(Y\alpha)^T/(4\rho)$ be $VDV^T$, we get $K^*=VD_+V^T$, and with $v_i$ the $i^{th}$ column of $V$, we can write
\BEAS
\Tr(K^*(Y\alpha)(Y\alpha)^T)&=&(Y\alpha)^TVD_+V^T(Y\alpha)\\
&=&\sum_i{\max\left(0,\lambda_i\left(K_0+\frac{1}{4\rho}(Y\alpha)(Y\alpha)^T\right)\right)(\alpha^TYv_i)^2}.
\EEAS
Using the same technique, we can also rewrite the term $\|K^*-K_0\|^2_F$ using this eigenvalue decomposition. Our original optimization problem (\ref{eq:kern_learn}) finally becomes
\setlength{\extrarowheight}{1ex}
\BEQ \label{eq:kern_learn_eig} \BA{ll}
\mbox{maximize} & \alpha^Te
-\frac{1}{2}\sum_i{\max(0,\lambda_i(K_0+(Y\alpha)(Y\alpha)^T/4\rho))(\alpha^TYv_i)^2}\\
&+\rho\sum_i{(\max(0,\lambda_i(K_0+(Y\alpha)(Y\alpha)^T/4\rho)))^2}\\
&-2\rho\sum_i\Tr((v_iv_i^T)K_0){\max(0,\lambda_i(K_0+(Y\alpha)(Y\alpha)^T/4\rho))}
+\rho\Tr(K_0K_0)\\
\mbox{subject to} & \alpha^Ty=0,0\leq\alpha\leq C
\EA\EEQ
in the variable $\alpha\in\reals^n$. By construction, the objective function is concave, hence (\ref{eq:kern_learn_eig}) is a convex optimization problem in $\alpha$.

A reformulation of problem (\ref{eq:kern_learn}) appears in \citeasnoun{Chen2008} where the authors move the inner minimization problem to the constraints and get the following semi-infinite quadratically constrained linear program (SIQCLP):
\BEQ \label{eq:SIQCLP} \BA{lll}
\mbox{maximize} & t & \\
\mbox{subject to} & \alpha^Ty=0,0\leq\alpha\leq C & \\
& t\leq \alpha^Te-
\frac{1}{2}\Tr(K(Y\alpha)(Y\alpha)^T)+\rho\|K-K_0\|^2_F & \forall K\succeq 0 .
\EA \EEQ
In Section \ref{s:algos}, we describe algorithms to solve our eigenvalue optimization problem in (\ref{eq:kern_learn_eig}), as well as an algorithm from \citeasnoun{Chen2008} that solves the different formulation in (\ref{eq:SIQCLP}), for completeness.

\subsection{Interpretation}
\label{ss:interpretation}
Our explicit solution of the optimal kernel given in (\ref{eq:kstar}) is the projection of a penalized rank-one update to the indefinite kernel on the cone of positive semidefinite matrices.  As $\rho$ tends to infinity, the rank-one update has less effect and in the limit, the optimal kernel is given by zeroing out the negative eigenvalues of the indefinite kernel.  This means that if the indefinite kernel contains a very small amount of noise, the best positive semidefinite kernel to use with SVM in our framework is the positive part of the indefinite kernel.

This limit as $\rho$ tends to infinity also motivates a heuristic for transforming the kernel on the testing set. Since negative eigenvalues in the training kernel are thresholded to zero in the limit, the same transformation should occur for the test kernel.  Hence, to measure generalization performance, we update the entries of the full kernel corresponding to training instances by the rank-one update resulting from the optimal solution to (\ref{eq:kern_learn_eig}) and threshold the negative eigenvalues of the full kernel matrix to zero to produce a Mercer kernel on the test set.

\subsection{Dual problem}
\label{ss:dual} As discussed above, problems (\ref{eq:kern_learn_hard}) and (\ref{eq:kern_learn}) are dual.  The inner maximization in problem~(\ref{eq:kern_learn_hard}) is a quadratic program in $\alpha$, whose dual is the quadratic minimization problem
\BEQ \label{eq:svm_dual} \BA{ll}
\mbox{minimize} & \frac{1}{2}(\-e-\delta+\mu+y\nu)^T(YKY)^{-1}(\-e-\delta+\mu+y\nu)+C\mu^Te\\
\mbox{subject to}&\delta,\mu \geq 0.\\
\EA \EEQ
Substituting (\ref{eq:svm_dual}) for the inner maximization in problem~(\ref{eq:kern_learn_hard}) allows us to write a joint minimization problem
\BEQ \label{eq:kern_learn_dual} \BA{ll}
\mbox{minimize} & \Tr(K^{-1}(Y^{-1}(\-e-\delta+\mu+y\nu))(Y^{-1}(\-e-\delta+\mu+y\nu))^T)/2 + C\mu^Te+\rho\|K-K_0\|^2_F\\
\mbox{subject to}& K\succeq0, \delta,\mu \geq 0\\
\EA \EEQ
in the variables $K\in\symm^n$, $\delta,\mu\in\reals^n$ and $\nu\in\reals$.  This is a quadratic program in the variables $\delta$, $\mu$ (which correspond to the constraints $0\leq\alpha\leq C$) and $\nu$ (which is the dual variable for the constraint $\alpha^Ty=0$). As we have seen earlier, any feasible solution $\alpha\in\reals^n$ produces a corresponding proxy kernel in (\ref{eq:kstar}). Plugging this kernel into problem~(\ref{eq:kern_learn_dual}) allows us to compute an upper bound on the optimum value of problem (\ref{eq:kern_learn}) by solving a simple quadratic program in the variables $\delta,~\mu,~\nu$. This result can then be used to bound the duality gap in (\ref{eq:kern_learn_eig}) and track convergence.

\section{Algorithms}
\label{s:algos}
We now detail two algorithms that can be used to solve problem (\ref{eq:kern_learn_eig}), which maximizes a nondifferentiable concave function subject to convex constraints.  An optimal point always exists since the feasible set is bounded and nonempty. For numerical stability, in both algorithms, we quadratically smooth our objective to compute a gradient. We first describe a simple projected gradient method which has numerically cheap iterations but less predictable performance in practice. We then show how to apply the analytic center cutting plane method, whose iterations are numerically more complex but which converges linearly.  For completeness, we also describe an exchange method from \citeasnoun{Chen2008} used to solve problem (\ref{eq:SIQCLP}), where the numerical bottleneck is a quadratically constrained linear program solved at each iteration.

\paragraph{Smoothing}  Our objective contains terms of the form $\max\{0,f(x)\}$ for some function $f(x)$, which are not differentiable (described in the section below). These functions are easily smoothed out by a Moreau-Yosida regularization technique (see \citeasnoun{Hiri96}, for example). We replace the max by a continuously differentiable $\frac{\epsilon}{2}$-approximation as follows:
\[ \varphi_\epsilon(f(x))=\max _{0 \le u\le 1}(uf(x)-\frac{\epsilon}{2}u^2).\]
The gradient is then given by $\nabla\varphi_\epsilon(f(x))=u^*(x)\nabla f(x)$ where $u^*(x)=\argmax\varphi_\epsilon(f(x))$.

\paragraph{Gradient}
Calculating the gradient of the objective function in (\ref{eq:kern_learn_eig}) requires computing the eigenvalue decomposition of a matrix of the form $X(\alpha)=K+\rho\alpha\alpha^T$. Given a matrix $X(\alpha)$, the derivative of the $i^{th}$ eigenvalue with respect to $\alpha$ is then given by
\BEQ
 \frac{\partial \lambda_i(X(\alpha))}{\partial \alpha} = v_i^T \frac{\partial X(\alpha)}{\partial \alpha}v_i
\EEQ
where $v_i$ is the $i^{th}$ eigenvector of $X(\alpha)$. We can then combine this expression with the smooth approximation above to obtain the gradient.

\subsection{Computing proxy kernels}
Because the proxy kernel in (\ref{eq:kstar}) only requires a rank one update of a (fixed) eigenvalue decomposition
\[
K^*=(K_0+(Y\alpha)(Y\alpha)^T/(4\rho))_+ ,
\]
we now briefly recall how $v_i$ and $\lambda_i(X(\alpha))$ can be computed efficiently in this case (see \citeasnoun{Demm1997} for further details).  We refer the reader to \citeasnoun{Kuli2006} for another kernel learning example using this method. Given the eigenvalue decomposition $X=VDV^T$, by changing basis this problem can be reduced to the decomposition of the diagonal plus rank-one matrix, $D+\rho uu^T$, where $u=V^T\alpha$. First, the updated eigenvalues are determined by solving the secular equations
\[
\det(D+\rho uu^T-\lambda I)=0,
\]
which can be done in $O(n^2)$.  While there is an explicit solution for the eigenvectors corresponding to these eigenvalues, they are not stable because the eigenvalues are approximated.  This instability is circumvented by computing a vector $\hat{u}$ such that approximate eigenvalues $\lambda$ are the exact eigenvalues of the matrix $D+\rho \hat{u}\hat{u}^T$, then computing its stable eigenvectors explicitly, where both steps can be done in $O(n^2)$ time.  The key is that $D+\rho \hat{u}\hat{u}^T$ is close enough to our original matrix so that the eigenvalues and eigenvectors are stable approximations of the true values.  Finally, the eigenvectors of our original matrix are computed as $VW$, with $W$ as the stable eigenvectors of $D+\rho \hat{u}\hat{u}^T$. Updating the eigenvalue decomposition is reduced to an $O(n^2)$ procedure plus one matrix multiplication, which is then the complexity of one gradient computation.

We note that eigenvalues of symmetric matrices are not differentiable when some of them have multiplicities greater than one (see \citeasnoun{Over1992} for a discussion), but a subgradient can be used instead of the gradient in all the algorithms detailed here. \citeasnoun{Lewi1999} shows how to compute an approximate subdifferential of the k-th largest eigenvalue of a symmetric matrix. This can then be used to form a regular subgradient of the objective function in (\ref{eq:kern_learn_eig}) which is concave by construction.

\subsection{Projected gradient method}
The projected gradient method takes a steepest descent step, then projects the new point back onto the feasible region (see \citeasnoun{Bert1999}, for example). We choose an initial point $\alpha_0\in\reals^n$ and the algorithm proceeds as in Algorithm \ref{alg:proj_grad}.

\begin{algorithm}
\caption{Projected gradient method}
\begin{algorithmic} [1]
\STATE Compute ${\alpha}_{i+1}=\alpha_i+t\nabla f(\alpha_i)$.
\STATE Set $\alpha_{i+1}=p_A({\alpha}_{i+1})$.
\STATE If $\mbox{gap} \le \epsilon$ stop, otherwise go back to step 1.
\end{algorithmic}
\label{alg:proj_grad}
\end{algorithm}

Here, we have assumed that the objective function is differentiable (after smoothing). The method is only efficient if the projection step is numerically cheap. The complexity of each iteration then breaks down as follows:

\par \noindent \emph{Step 1.} This requires an eigenvalue decomposition that is computed in $O(n^2)$ plus one matrix multiplication as described above.  Experiments below use a stepsize of $5/k$ for IndefiniteSVM and $10/k$ for PerturbSVM (described in Section \ref{ss:learn-mercer}) where $k$ is the iteration number.  A good stepsize is crucial to performance, and must be chosen separately for each data set as there is no rule of thumb.  We note that a line search would be costly here because it would require multiple eigenvalue decompositions to recalculate the objective multiple times.

\par \noindent \emph{Step 2.} This is a projection onto the region $A=\{\alpha^Ty=0,~0\leq\alpha\leq C\}$ and can be solved explicitly by sorting the vector of entries, with cost $O(n\log n)$.

\par \noindent \emph{Stopping Criterion.} We can compute a duality gap using the results of \S\ref{ss:dual} where
\[
K_i=(K_0+(Y\alpha_i)(Y\alpha_i)^T/(4\rho))_+\]
is the candidate kernel at iteration $i$ and we solve problem (\ref{eq:gen_perf_meas}), which simply means solving a SVM problem with the positive semidefinite kernel $K_i$, and produces an upper bound on (\ref{eq:kern_learn_eig}), hence a bound on the suboptimality of the current solution.

\par \noindent \emph{Complexity.} The number of iterations required by this method to reach a target precision of $\epsilon$ grows as $O(1/\epsilon^2)$. See \citeasnoun{Nest03a} for a complete discussion.

\subsection{Analytic center cutting plane method}
The analytic center cutting plane method (ACCPM) reduces the feasible region at each iteration using a new \emph{cut} computed by evaluating a subgradient of the objective function at the analytic center of the current feasible set, until the volume of the reduced region converges to the target precision.  This method does not require differentiability.  We set $\mathcal{L}_0=\{x\in\reals^n\mid x^Ty=0,0\leq x\leq C\}$, which we can write as $\{x\in\reals^n\mid A_0x\leq b_0\}$, to be our first localization set for the optimal solution.  The method is described in Algorithm \ref{alg:accpm} (see \citeasnoun{Bert1999} for a more complete treatment of cutting plane methods).

\begin{algorithm}
\caption{Analytic center cutting plane method}
\begin{algorithmic} [1]
\STATE Compute $\alpha_i$ as the analytic center of $\mathcal{L}_i$ by solving
    \[
    x_{i+1}=\argmin_{x\in\reals^n} ~ -\sum_{i=1}^m \mbox{log}(b_i-a_i^Tx)
    \]
    where $a_i^T$ represents the $i^{th}$ row of coefficients from the left-hand side of $\{x\in\reals^n\mid A_ix\leq b_0\}$.
\STATE Compute $\nabla f(x)$ at the center $x_{i+1}$ and update the (polyhedral) localization set
    \[
    \mathcal{L}_{i+1}=\mathcal{L}_i \cap \{\nabla f(x_{i+1})(x-x_{i+1}) \ge 0\}
    \]
    where $f$ is objective in problem (\ref{eq:kern_learn_eig}).
\STATE If $m\ge 3n$, reduce the number of constraints to $3n$.
\STATE If $\mbox{gap} \le \epsilon$ stop, otherwise go back to step 1.
\end{algorithmic}
\label{alg:accpm}
\end{algorithm}

\par \noindent The complexity of each iteration breaks down as follows:

\par \noindent \emph{Step 1.} This step computes the analytic center of a polyhedron and can be solved in $O(n^3)$ operations using interior point methods, for example.

\par \noindent \emph{Step 2.} This simply updates the polyhedral description.  It includes the gradient computation which again is $O(n^2)$ plus one matrix multiplication.

\par \noindent \emph{Step 3.} This step requires ordering the constraints according to their relevance in the localization set.  One relevance measure for the $j^{th}$ constraint at iteration $i$ is
\BEQ
\frac{a_j^T\nabla^2f(x_i)^{-1}a_j}{(a_j^Tx_i-b_j)^2}
\EEQ
where $f$ is the objective function of the analytic center problem. Computing the hessian is easy: it requires matrix multiplication of the form $A^TDA$ where $A$ is $m\times n$ (matrix multiplication is kept inexpensive in this step by pruning redundant constraints) and $D$ is diagonal.  Restricting the number of constraints to $3n$ is a rule of thumb; raising this limit increases the per iteration complexity while decreasing it increases the required number of iterations.

\par \noindent \emph{Stopping Criterion.}  An upper bound is computed by maximizing a first order Taylor approximation of $f(\alpha)$ at $\alpha_i$ over all points in an ellipsoid that covers $\mathcal{A}_i$, which can be computed explicitly.

\par \noindent \emph{Complexity.}  ACCPM is provably convergent in $O(n(\log 1/\epsilon)^2)$ iterations when using cut elimination, which keeps the complexity of the localization set bounded. Other schemes are available with slightly different complexities: a bound of $O(n^2/\epsilon^2)$ is achieved in \citeasnoun{Goff2002} using (cheaper) approximate centers, for example.

\subsection{Exchange method for SIQCLP}
The algorithm considered in \citeasnoun{Chen2008} in order to solve problem (\ref{eq:SIQCLP}) falls under a class of algorithms called exchange methods (as defined in \citeasnoun{Hett1993}).  These methods iteratively solve problems constrained by a finite subset of the infinitely many constraints, where the solution at each iterate gives an improved lower bound to the maximization problem.  The subproblem solved at each iteration here is
\BEQ \label{eq:SIQCLP_master} \BA{lll}
\mbox{maximize} & t & \\
\mbox{subject to} & \alpha^Ty=0,0\leq\alpha\leq C & \\
& t\leq \alpha^Te-
\frac{1}{2}\Tr(K_i(Y\alpha)(Y\alpha)^T)+\rho\|K_i-K_0\|^2_F & i=1,\ldots,p
\EA \EEQ
where $p$ is the number of constraints used to approximate the infinitely many constraints of problem (\ref{eq:SIQCLP}).  Let $(t_1,\alpha_1)$ be an initial solution found by solving problem (\ref{eq:SIQCLP_master}) with $p=1$ and $K_1=(K_0)_+$, where $K_0$ is the input indefinite kernel.  The algorithm proceeds as in Algorithm \ref{alg:exchange} below.

\begin{algorithm}
\caption{Exchange method}
\begin{algorithmic} [1]
\STATE Compute $K_{i+1}$ by solving the inner minimization problem of (\ref{eq:kern_learn}) as a function of $\alpha_i$.
\STATE Stop if
    \[\alpha_i^Te-
\frac{1}{2}\Tr(K_{i+1}(Y\alpha_i)(Y\alpha_i)^T)+\rho\|K_{i+1}-K_0\|^2_F\ge t_i .\]
\STATE Solve problem (\ref{eq:SIQCLP_master}) with an additional constraint using $K_{i+1}$ to get $(t_{i+1},\alpha_{i+1})$ and go back to step 1.
\end{algorithmic}
\label{alg:exchange}
\end{algorithm}

\par \noindent The complexity of each iteration breaks down as follows:

\par \noindent \emph{Step 1.} This step can be solved analytically using Theorem \ref{th:kstar}.  An efficient calculation of $K_{i+1}$ can be made as in the other algorithms above using an $O(n^2)$ procedure plus one matrix multiplication.

\par \noindent \emph{Step 2 (Stopping Criterion).} The previous point $(t_{i},\alpha_{i})$ is optimal if it is feasible with respect to the new constraint, in which case it is feasible for the infinitely many constraints of the original problem (\ref{eq:SIQCLP}) and hence also optimal.

\par \noindent \emph{Step3.} This step requires solving a QCLP with a number of quadratic constraints equivalent to the number of iterations.  As shown in \citeasnoun{Chen2008}, the QCLP can be written as a regularized version of the multiple kernel learning (MKL) problem from \citeasnoun{Lanc2004}, where the number of constraints here is equivalent to the number of kernels in MKL.  Efficient methods to solve MKL with many kernels is an active area of research, most recently in \citeasnoun{Rako2007}.  There, the authors use a gradient method to solve a reformulation of problem (\ref{eq:SIQCLP_master}) as a smooth maximization problem.  Each objective value and gradient computation requires computing a support vector machine, hence each iteration requires several SVM computations which can be speeded up using warm-starting.  Furthermore, \citeasnoun{Chen2008} prune inactive constraints at each iteration in order to decrease the number of constraints in the QCLP.

\par \noindent \emph{Complexity.}  No rate of convergence is known for this algorithm, but the duality gap given in \citeasnoun{Chen2008} is shown to monotonically decrease.

\subsection{Matlab Implementation}
The first two algorithms discussed here were implemented in Matlab for the cases of indefinite (IndefiniteSVM) and positive semidefinite (PerturbSVM) kernels and can be downloaded from the authors' webpages in a package called IndefiniteSVM.  The $\rho$ penalty parameter is one-dimensional in the implementation. This package makes use of the  LIBSVM code of \citeasnoun{CC01a} to produce suboptimality bounds and track convergence.  A Matlab implementation of the exchange method (due to the authors of \citeasnoun{Chen2008}) that uses MOSEK \cite{Mosek} to solve problem (\ref{eq:SIQCLP_master}) is compared against the projected gradient method in Section \ref{s:num-res}.

\section{Extensions}
\label{s:extensions} In this section, we extend our results to other kernel methods, namely support vector regressions and one-class support vector machines.  In addition, we apply our method to using Mercer kernels and show how to use more general penalties in our formulation.

\subsection{SVR with indefinite kernels}
\label{ss:svr-indkernels} The practicality of indefinite kernels in SVM classification similarly motivates using indefinite kernels in support vector regression (SVR).  We here extend the formulations in Section \ref{s:max-min} to SVR with linear $\epsilon$-insensitive loss
\BEQ \label{eq:gen_perf_meas_SVR}
\omega_C(K)=\max_{\{-C\leq\alpha\leq C,\alpha^Te=0\}}\alpha^Ty-\epsilon|\alpha |-\Tr(K\alpha\alpha^T)/2
\EEQ
where $\alpha\in\reals^n$ and $C$ is the SVR penalty parameter.  The indefinite SVR formulation follows directly as in Section \ref{ss:indKernelLearning} and the optimal kernel is learned by solving
\BEQ \label{eq:kern_learn_SVR}
\max_{\{\alpha^Te=0,-C\leq\alpha\leq C\}} ~ \min_{\{K\succeq0\}} ~ \alpha^Ty-\epsilon|\alpha |-
\frac{1}{2}\Tr(K\alpha\alpha^T)+\rho\|K-K_0\|^2_F
\EEQ
in the variables $K\in\symm^n$ and $\alpha\in\reals^n$, where the parameter $\rho >0$ controls the magnitude of the penalty on the distance between $K$ and $K_0$.  The following corollary to Theorem \ref{th:kstar} provides the solution to the inner minimization problem in (\ref{eq:kern_learn_SVR})
\begin{corollary} \label{cor:kstar_SVR}
Given a similarity matrix $K_0\in S^n$ and a vector $\alpha\in\reals^n$ of support vector coefficients, the optimal kernel in problem (\ref{eq:kern_learn_SVR}) can be computed explicitly as
\BEQ\label{eq:kstar_SVR}
K^*=(K_0+\alpha\alpha^T/(4\rho))_+
\EEQ
where $\rho\geq 0$ controls the penalty.
\end{corollary}
The proof follows directly as in Theorem \ref{th:kstar}; the slight difference is that the vector of labels $y$ does not appear in the optimal kernel.  Plugging in (\ref{eq:kstar_SVR}) into (\ref{eq:kern_learn_SVR}), the resulting formulation can be rewritten as the convex eigenvalue optimization problem
\setlength{\extrarowheight}{1ex}
\BEQ \label{eq:kern_learn_eig_SVR} \BA{ll}
\mbox{maximize} & \alpha^Ty-\epsilon|\alpha |
-\frac{1}{2}\sum_i{\max(0,\lambda_i(K_0+\alpha\alpha^T/(4\rho)))(\alpha^Tv_i)^2}\\
&+\rho\sum_i{(\max(0,\lambda_i(K_0+\alpha\alpha^T/4\rho)))^2}\\
&-2\rho\sum_i\Tr((v_iv_i^T)K_0){\max(0,\lambda_i(K_0+\alpha\alpha^T/(4\rho)))}
+\rho\Tr(K_0K_0)\\
\mbox{subject to} & \alpha^Te=0,-C\leq\alpha\leq C\\
\EA\EEQ
in the variable $\alpha\in\reals^n$.  Again, a proxy kernel given by (\ref{eq:kstar_SVR}) can be produced  from any feasible solution $\alpha\in\reals^n$.  Plugging the proxy kernel into problem~(\ref{eq:kern_learn_SVR}) allows us to compute an upper bound on the optimum value of problem (\ref{eq:kern_learn_SVR}) by solving a support vector regression problem.

\subsection{One-class SVM with indefinite kernels}
\label{ss:1class_svm-indkernels}  The same reformulation can also be applied to one-class support vector machines which have the formulation (see \citeasnoun{Scho2002})
\BEQ \label{eq:gen_perf_meas_1class_svm}
\omega_\nu(K)=\max_{\{0\leq\alpha\leq \frac{1}{\nu l},\alpha^Te=1\}}-\Tr(K\alpha\alpha^T)/2
\EEQ
where $\alpha\in\reals^n$, $\nu$ is the one-class SVM parameter, and $l$ is the number of training points.  The indefinite one-class SVM formulation follows again as done for binary SVM and SVR; the optimal kernel is learned by solving
\BEQ \label{eq:kern_learn_1class_SVM}
\max_{\{\alpha^Te=1,0\leq\alpha\leq \frac{1}{\nu l}\}} ~ \min_{\{K\succeq0\}} ~ -
\frac{1}{2}\Tr(K\alpha\alpha^T)+\rho\|K-K_0\|^2_F
\EEQ
in the variables $K\in\symm^n$ and $\alpha\in\reals^n$.  The inner minimization problem is identical to that of indefinite SVR and the optimal kernel has the same form as given in Corollary \ref{cor:kstar_SVR}.  Plugging (\ref{eq:kstar_SVR}) into (\ref{eq:kern_learn_1class_SVM}) gives another convex eigenvalue optimization problem
\setlength{\extrarowheight}{1ex}
\BEQ \label{eq:kern_learn_eig_1class_SVM} \BA{ll}
\mbox{maximize} &
-\frac{1}{2}\sum_i{\max(0,\lambda_i(K_0+\alpha\alpha^T/4\rho))(\alpha^Tv_i)^2}\\
&+\rho\sum_i{(\max(0,\lambda_i(K_0+\alpha\alpha^T/(4\rho))))^2}\\
&-2\rho\sum_i\Tr((v_iv_i^T)K_0){\max(0,\lambda_i(K_0+\alpha\alpha^T/(4\rho)))}
+\rho\Tr(K_0K_0)\\
\mbox{subject to} & \alpha^Te=1,0\leq\alpha\leq \frac{1}{\nu l}\\
\EA\EEQ
in the variable $\alpha\in\reals^n$, which is identical to (\ref{eq:kern_learn_eig_SVR}) without the first two terms in the objective and slightly different constraints.  The algorithm follows almost directly the same as above for the indefinite SVR formulation.

%
%

\subsection{Learning from Mercer kernels}
\label{ss:learn-mercer} While our central motivation is to use indefinite kernels for SVM classification, one would also like to analyze what happens when a Mercer kernel is used as input in (\ref{eq:kern_learn}). In this case, we learn another kernel that decreases the upper bound on generalization performance and produces perturbed support vectors.  We can again interpret the input as a noisy kernel, and as such, one that will achieve suboptimal performance.  If the input kernel is the best kernel to use (i.e. is not noisy), we will observe that our framework achieves optimal performance as $\rho$ tends to infinity (through cross validation), otherwise we simply learn a better kernel using a finite $\rho$.

When the similarity measure $K_0$ is positive semidefinite, the proxy kernel $K^*$ in Theorem \ref{th:kstar} simplifies to a rank-one update of $K_0$
\BEQ\label{eq:kstar_psd}
K^*=K_0+(Y\alpha^*)(Y\alpha^*)^T/(4\rho)
\EEQ
whereas, for indefinite $K_0$, the solution was to project this matrix on the cone of positive semidefinite matrices.  Plugging (\ref{eq:kstar_psd}) into problem (\ref{eq:kern_learn}) gives:
\BEQ\label{eq:learn-mercer}
\DS \max_{\{\alpha^Ty=0,~0\leq\alpha\leq C\}} ~ \alpha^Te
-\frac{1}{2}\Tr(K_0(Y\alpha)(Y\alpha)^T)-\frac{1}{16\rho}\sum_{i,j}{(\alpha_i\alpha_j)^2} ,
\EEQ
which is the classic SVM problem given in (\ref{eq:gen_perf_meas}) with a fourth order penalty on the support vectors. For testing in this framework, we do not need to transform the kernel, only the support vectors are perturbed. In this case, computing the gradient no longer requires eigenvalue decompositions at each iteration.  Experimental results are shown in Section \ref{s:num-res}.

\subsection{Componentwise penalties}
Indefinite SVM can be generalized further with componentwise penalties on the distance between the proxy kernel and the indefinite kernel $K_0$.  We generalize problem (\ref{eq:kern_learn}) to
\BEQ \label{eq:kern_learn_pen}
\max_{\{\alpha^Ty=0,0\leq\alpha\leq C\}} ~ \min_{\{K\succeq0\}} ~ \alpha^Te-
\frac{1}{2}\Tr(K(Y\alpha)(Y\alpha)^T)+\sum_{i,j}H_{ij}(K_{ij}-K_{0ij})^2
\EEQ
where $H$ is now a matrix of varying penalties on the componentwise distances.  For a specific class of penalties, the optimal kernel $K^*$ can be derived explicitly as follows.
\begin{theorem} \label{th:kstar_pen}
Given a similarity matrix $K_0\in S^n$, a vector $\alpha\in\reals^n$ of support vector coefficients and the label matrix $Y=\diag(y)$, when $H$ is rank-one with $H_{ij}=h_ih_j$, the optimal kernel in problem (\ref{eq:kern_learn_pen}) has the explicit form
\BEQ\label{eq:kstar_pen}
K^*=W^{-1/2}((W^{1/2}(K_0+\frac{1}{4}(W^{-1}Y\alpha^*)(W^{-1}Y\alpha^*)^T)W^{1/2})_+)W^{-1/2}
\EEQ
where $W$ is the diagonal matrix with $W_{ii}=h_i$.
\end{theorem}
\begin{proof}
The inner minimization problem to problem (\ref{eq:kern_learn_pen}) can be written out as
\[
\min_{\{K\succeq0\}} ~
\sum_{i,j}H_{ij}(K_{ij}^2-2K_{ij}K_{0ij}+K_{0ij}^2)-\frac{1}{2}\sum_{i,j}y_iy_j\alpha_i\alpha_jK_{i,j} .
\]
Adding and subtracting $\sum_{i,j}H_{ij}(K_{0ij}+\frac{1}{4H_{ij}}y_iy_j\alpha_i\alpha_j)^2$, combining similar terms, and removing remaining constants gives
\[ \BA {ll}
\mbox{minimize} & \|H^{1/2}\circ(K-(K_0+\frac{1}{4H}\circ(Y\alpha)(Y\alpha)^T))\|^2_F\\
\mbox{subject to} & K\succeq0
\EA \]
where $\circ$ denotes the Hardamard product, $(A\circ B)_{ij}=a_{ij}b_{ij}$, $(H^{1/2})_{ij}=H_{ij}^{1/2}$, and $(\frac{1}{4H})_{ij}=\frac{1}{4H_{ij}}$. This is a weighted projection problem where $H_{ij}$ is the penalty on $(K_{ij}-K_{0ij})^2$.  Since $H$ is rank-one, the result follows from Theorem 3.2 of \citeasnoun{High2002}.
\end{proof}

Notice that Theorem \ref{th:kstar_pen} is a generalization of Theorem \ref{th:kstar} where we had $H=ee^T$.  In constructing a rank-one penalty matrix $H$, we simply assign penalties to each training point.   The componentwise penalty formulation can also be extended to true kernels.  If $K_0\succeq0$, then $K^*$ in Theorem \ref{th:kstar_pen} simplifies to a rank-one update of $K_0$:
\BEQ\label{eq:kstar_pen_psd}
K^*=K_0+\frac{1}{4}(W^{-1/2}Y\alpha)(W^{-1/2}Y\alpha)^T
\EEQ
where no projection is required.

\section{Experiments}
\label{s:num-res}
In this section we compare the generalization performance of our technique to other methods applying SVM classification to indefinite similarity measures. We also examine classification performance using Mercer kernels. We conclude with experiments showing convergence of our algorithms. All experiments on Mercer kernels use the LIBSVM library.

\subsection{Generalization with indefinite kernels}
We compare our method for SVM classification with indefinite kernels to several  kernel preprocessing techniques discussed earlier. The first three techniques perform spectral transformations on the indefinite kernel.  The first, called \emph{denoise} here, thresholds the negative eigenvalues to zero.  The second transformation, called \emph{flip}, takes the absolute value of all eigenvalues.  The last transformation, \emph{shift}, adds a constant to each eigenvalue, making them all positive.  See \citeasnoun{Wu2005} for further details.  We also implemented an SVM modification (denoted \emph{Mod SVM}) suggested in \citeasnoun{Lin2003} where a nonconvex quadratic objective function is made convex by replacing the indefinite kernel with the identity matrix.  The kernel only appears in linear inequality constraints that separate the data.  Finally, we compare our results with a direct use of SVM classification on the original indefinite kernel (SVM converges but the solution is only a stationary point and not guaranteed to be optimal).

We first experiment on data from the USPS handwritten digits database  \citeasnoun{Hul1994} using the indefinite Simpson score and the one-sided tangent distance kernel to compare two digits.  The tangent distance is a transformation invariant measure---it assigns high similarity between an image and slightly rotated or shifted instances---and is known to perform very well on this data set.  Our experiments symmetrize the one-sided tangent distance using the square of the mean tangent distance defined in \citeasnoun{Haas2002} and make it a similarity measure by negative exponentiation.  We also consider the Simpson score for this task, which is much cheaper to compute (a ratio comparing binary pixels).  We finally analyze three data sets (diabetes, german and ala) from the UCI repository \cite{Asun2007} using the indefinite sigmoid kernel.

The data is randomly divided into training and testing data.  We apply 5-fold cross validation and use an average of the accuracy and recall measures (described below) to determine the optimal parameters $C$, $\rho$, and any kernel inputs.  We then train a model with the full training set and optimal parameters and test on the independent test set.

\begin{table}[h!]
\BC
\small{
\begin{tabular}{|c|c|c|c|c|}
\hline
  Data Set & \# Train & \# Test & $\lambda_{min}$ & $\lambda_{max}$\\
\hline
USPS-3-5-SS & 767 & 773 & -70.00 & 903.94 \\
\hline
USPS-3-5-TD1 & 767 & 773 & -0.31 & 764.72 \\
\hline
USPS-4-6-SS & 829 & 857 & -74.38 & 819.36 \\
\hline
USPS-4-6-TD1 & 829 & 857 & -0.72 & 771.07 \\
\hline
diabetes-sig & 384 & 384  & -.65 & 211.62 \\
\hline
german-sig & 500 & 500 & -928.10 & 8.50 \\
\hline
a1a-sig & 803 & 802 & -.01 & 84.44 \\
\hline
\end{tabular}
}
\EC
\caption{Summary statistics for the various data sets used in our experiments. The USPS data comes from the USPS handwritten digits database, the other data sets are taken from the UCI repository.  \textit{SS} refers to the Simpson kernel, \textit{TD1} to the one-sided tangent distance kernel, and \textit{sig} to the sigmoid kernel.  Training and testing sets were divided randomly.  Notice that the Simpson kernels are mostly highly indefinite while the one-sided tangent distance kernel is nearly positive semidefinite.  The sigmoid kernel is highly indefinite depending on the parametrization.  Statistics for sigmoid kernels refer to the optimal kernel parameterized under cross validation with Indefinite SVM. Spectrums are based on the full kernel, i.e. combining training and testing data.}

\label{table:data_stats}
\end{table}

\begin{table}[h!]
\begin{center}
\small{
\begin{tabular}{|c|c|c|c|c|c|c|c|}
\hline
  Data Set & Measure & Denoise & Flip & Shift & Mod SVM & SVM & Indefinite SVM \\
\hline
 \multirow{3}{*}{USPS-3-5-SS}
&Accuracy & 95.47 & 95.21 & 93.27 & \textbf{96.12} & 69.47 & 95.73\\
&Recall & 94.50 & 94.50 & 94.98 & 96.17 & 67.94 & \textbf{97.13}\\
&Average & 94.98 & 94.86 & 94.12 & 96.15 & 68.71 & \textbf{96.43}\\
\hline
 \multirow{3}{*}{USPS-3-5-TD1}
&Accuracy & \textbf{98.58} & 98.45 & \textbf{98.58} & 98.19 & \textbf{98.58} & 98.45\\
&Recall & \textbf{98.56} & 98.33 & \textbf{98.56} & 97.85 & \textbf{98.56} & 98.33\\
&Average & \textbf{98.57} & 98.39 & \textbf{98.57} & 98.02 & \textbf{98.57} & 98.39\\
\hline
 \multirow{3}{*}{USPS-4-6-SS}
&Accuracy & \textbf{98.60} & 98.25 & 96.73 & \textbf{98.60} & 84.36 & 98.25\\
&Recall & 99.32 & 99.32 & 96.61 & 99.32 & 81.72 & \textbf{99.77}\\
&Average & 98.96 & 98.79 & 96.67 & 98.96 & 83.04 & \textbf{99.01}\\
\hline
 \multirow{3}{*}{USPS-4-6-TD1}
&Accuracy & \textbf{99.30} & \textbf{99.30} & 99.18 & 99.18 & \textbf{99.30} & \textbf{99.30}\\
&Recall & \textbf{99.77} & \textbf{99.77} & 99.55 & 99.55 & \textbf{99.77} & \textbf{99.77}\\
&Average & \textbf{99.54} & \textbf{99.54} & 99.37 & 99.37 & \textbf{99.54} & \textbf{99.54}\\
\hline
 \multirow{3}{*}{diabetes-sig}
&Accuracy & 74.48 & 74.74 & 76.56 & 76.04 & 73.70 & \textbf{77.08}\\
&Recall & 78.40 & 76.80 & \textbf{89.60} & 78.40 & 76.40 & 89.20\\
&Average & 76.44 & 75.77 & 83.08 & 77.22 & 75.05 & \textbf{83.14}\\
\hline
 \multirow{3}{*}{german-sig}
&Accuracy & 70.40 & 70.40 & \textbf{75.60} & 72.60 & 69.40 & 62.80\\
&Recall & 78.00 & 78.00 & 46.67 & 66.00 & 80.00 & \textbf{85.33}\\
&Average & 74.20 & 74.20 & 61.13 & 69.30 & \textbf{74.70} & 74.07\\
\hline
 \multirow{3}{*}{a1a-sig}
&Accuracy & 74.06 & 76.18 & 75.69 & 78.55 & 75.69 & \textbf{82.92}\\
&Recall & 87.31 & 87.82 & 87.31 & \textbf{89.34} & 87.82 & 81.73\\
&Average & 80.69 & 82.00 & 81.50 & \textbf{83.95} & 81.75 & 82.32\\
\hline
\end{tabular}
}
\end{center}
\caption{Indefinite SVM performs favorably for the highly indefinite Simpson kernels.  Performance is comparable for the nearly positive semidefinite one-sided tangent distance kernel.  Comparable performance with sigmoid kernels is more consistent with indefinite SVM across data sets.  The performance measures are: $\mbox{Accuracy}=\frac{TP+TN}{TP+TN+FP+FN}$, $\mbox{Recall}=\frac{TP}{TP+FN}$, and $\mbox{Average}=(\mbox{Accuracy}+\mbox{Recall})/2$.}
\label{table:gen_performance}
\end{table}

Table \ref{table:data_stats} provides summary statistics for these data sets, including the minimum and maximum eigenvalues of the training similarity matrices. We observe that the Simpson are highly indefinite, while the one-sided tangent distance kernel is nearly positive semidefinite.  The spectrum of sigmoid kernels varies greatly across examples because it is very sensitive to the sigmoid kernel parameters.  Table \ref{table:gen_performance} compares accuracy, recall, and their average for denoise, flip, shift, modified SVM, direct SVM and the indefinite SVM algorithm described in this work.

Based on the interpretation from Section \ref{ss:interpretation}, Indefinite SVM should be expected to perform at least as well as denoise; if denoise were a good transformation, then cross-validation over $\rho$ should choose a high penalty that makes Indefinite SVM and denoise nearly equivalent.  The rank-one update provides more flexibility for the transformation and similarities concerning data points $x_i$ that are easily classified ($\alpha_i=0$) are not modified by the rank-one update.  Further interpretation for the specific rank-one update is not currently known.  However, \citeasnoun{Chen2009} recently proposed spectrum modifications in a similar manner to Indefinite SVM.  Rather than perturb the entire indefinite similarity matrix, they perturb the spectrum directly allowing improvements over the denoise as well as flip transformations.  They also note that Indefinite SVM might perform better on sparse kernels because the rank-one update may then allow inference of hidden relationships.

We observe that Indefinite SVM performs comparably on all USPS examples (slightly better for the highly indefinite Simpson kernels), which are relatively easy classification problems.  As expected, classification using the tangent distance outperforms classification with the Simpson score but, as mentioned above, the Simpson score is cheaper to compute.  We also note that other documented classification results on this USPS data set perform multi-classification, while here we only perform binary classification.  Classification of the UCI data sets with sigmoid kernels is more difficult (as demonstrated by lower performance measures).  Indefinite SVM here is the only technique that outperforms in at least one of the measures across all three data sets.

\subsection{Generalization with Mercer kernels}
Using this time linear and gaussian (both positive semidefinite, i.e. Mercer) kernels on the USPS data set, we now compare classification performance using regular SVM and the penalized kernel learning problem (\ref{eq:learn-mercer}) of Section \ref{ss:learn-mercer}, which we call PerturbSVM here. We also test these two techniques on positive semidefinite kernels formed using noisy USPS data sets (created by adding uniformly distributed noise in [-1,1] to each pixel before normalizing to [0,1]), in which case PerturbSVM can be seen as optimally denoised support vector machine classification.  We again cross-validate on a training set and test on the same independent group of examples used in the experiments above.  Optimal parameters from classification of unperturbed data were used to train classifiers for perturbed data.  Results are summarized in Table \ref{table:gen_performance_perturbSVM}.

These results show that PerturbSVM performs at least as well in almost all cases. As expected, noise decreased generalization performance in all experiments.  Except in the \textit{USPS-4-6-gaussian} example, the value of $\rho$ selected was not the highest possible for each test where PerturbSVM outperforms SVM in at least one measure; this implies that the support vectors were perturbed to improve classification.  Overall, when zero or moderate noise is present, PerturbSVM does improve performance over regular SVM as shown. When too much noise is present however (for example, pixel data with range in [-1,1] was modified with uniform noise in [-2,2] before being normalized to [0,1]), the performance of both techniques is comparable.

\begin{table}[h!]
\begin{center}
\small{
\begin{tabular}{|c|c|c|c|c|c|}
\hline
   & & \multicolumn{2}{|c|}{Unperturbed} & \multicolumn{2}{|c|}{Noisy} \\ \hline
  Data Set & Measure & SVM & Perturb SVM & SVM & Perturb SVM\\
\hline
 \multirow{3}{*}{USPS-3-5-linear}
&Accuracy & \textbf{96.25} & 96.12 & 90.27 & \textbf{93.16}\\
&Recall & 95.69 & \textbf{95.93} & 90.00 & \textbf{92.87}\\
&Average & 95.97 & \textbf{96.03} & 90.14 & \textbf{93.01}\\
\hline
 \multirow{3}{*}{USPS-4-6-linear}
&Accuracy & \textbf{99.07} & \textbf{99.07} & 97.39 & \textbf{97.97}\\
&Recall & 99.10 & \textbf{99.32} & 97.34 & \textbf{98.13}\\
&Average & 99.08 & \textbf{99.19} & 97.36 & \textbf{98.05}\\
\hline
 \multirow{3}{*}{USPS-3-5-gaussian}
&Accuracy & \textbf{97.67} & 97.54 & 92.11 & \textbf{93.57}\\
&Recall & \textbf{98.09} & 97.37 & 91.27 & \textbf{92.89}\\
&Average & \textbf{97.88} & 97.46 & 91.69 & \textbf{93.23}\\
\hline
 \multirow{3}{*}{USPS-4-6-gaussian}
&Accuracy & 99.18 & \textbf{99.30} & \textbf{98.00} & 97.99\\
&Recall & \textbf{99.55} & \textbf{99.55} & 98.15 & \textbf{98.19}\\
&Average & 99.37 & \textbf{99.42} & 98.08 & \textbf{98.09}\\
\hline
\end{tabular}
}
\end{center}
\caption{Performance measures for USPS data using linear and gaussian kernels. \textit{Unperturbed} refers to classification of the original data and \textit{Noisy} refers to classification of data that is perturbed by uniform noise. Perturb SVM perturbs the support vectors to improve generalization.  However, performance is lower for both techniques in the presence of high noise.
}
\label{table:gen_performance_perturbSVM}
\end{table}

\subsection{Convergence}
We ran our two algorithms on data sets created by randomly perturbing the four USPS data sets used above. Average results and standard deviation are displayed in Figure \ref{fig:convergence_figures} in semilog scale (note that the codes were not stopped here and that the target duality gap improvement is usually much smaller than $10^{-8}$). As expected, ACCPM converges much faster (in fact linearly) to a higher precision, while each iteration requires solving a linear program of size $n$. The gradient projection method converges faster in the beginning but stalls at higher precision, however each iteration only requires a rank one update on an eigenvalue decomposition.

\begin{figure} [h!]
\begin{tabular*}{\textwidth}{@{\extracolsep{\fill}}cc}
\psfrag{gap}[b][t]{\small{Duality Gap}}
\psfrag{iter}[t][b]{\small{Iteration}}
\psfrag{accpm}[b]{\small{ACCPM}}
\includegraphics[width=.50 \textwidth]{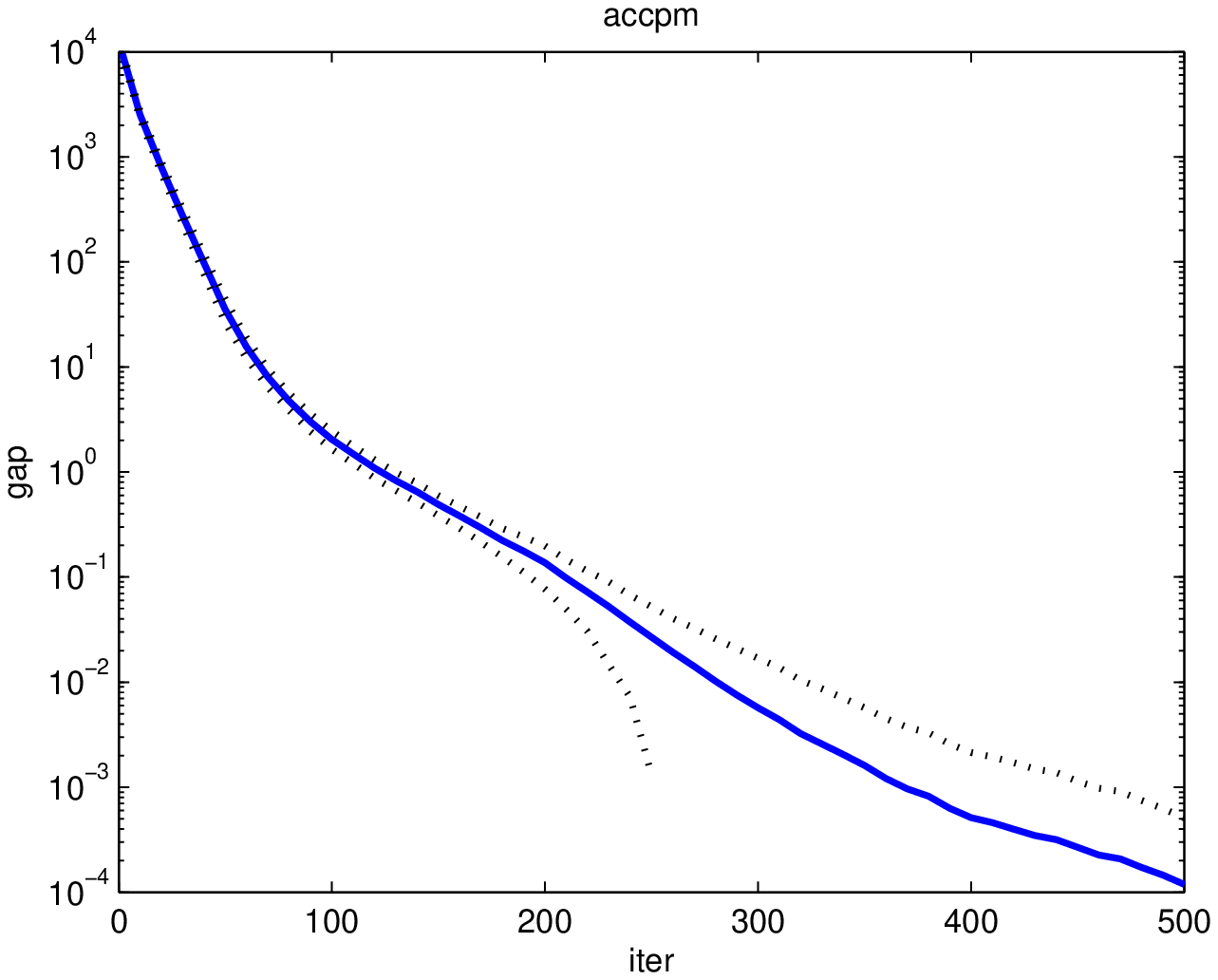} &
\psfrag{gap}[b][t]{\small{Duality Gap}} \psfrag{iter}[t][b]{\small{Iteration}}
\psfrag{pg}[b]{\small{Projected Gradient Method}}
\includegraphics[width=.50 \textwidth]{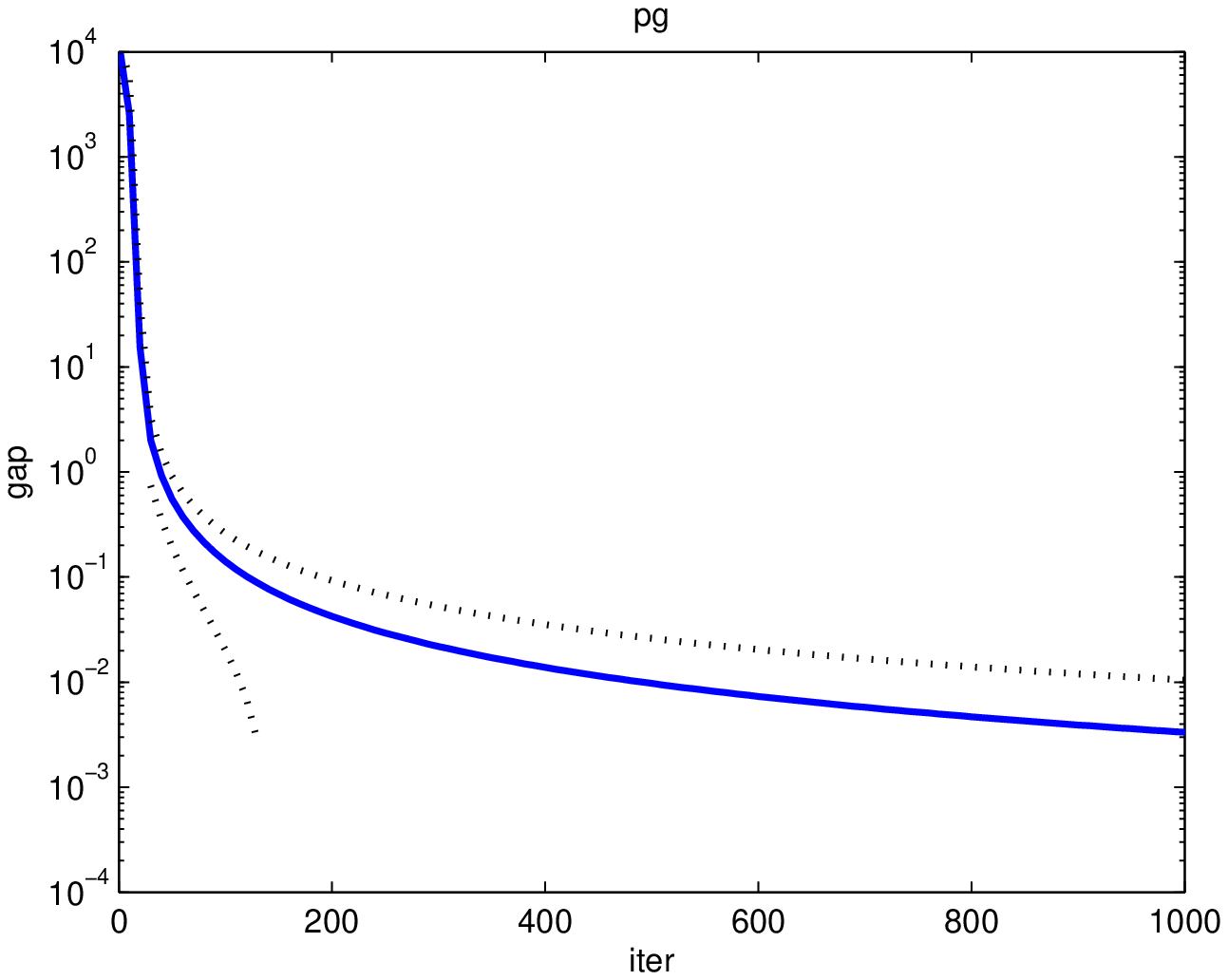}
\end{tabular*}
\caption{Convergence plots for ACCPM (left) and projected gradient method (right) on random subsets of the USPS-SS-3-5 data set  (average gap versus iteration number, dashed lines at plus and minus one standard deviation). ACCPM converges linearly to a higher precision while the gradient projection method converges faster in the beginning but stalls at a higher precision.}
\label{fig:convergence_figures}
\end{figure}

We finally examine the computing time of IndefiniteSVM using the projected gradient method and ACCPM and compare them with the SIQCLP method of \citeasnoun{Chen2008}.  Figure \ref{fig:time_figures} shows total runtime (left) and average iteration runtime (right) for varying problem dimensions on an example from the USPS data with Simpson kernel.  Experiments are averaged over 10 random data subsets and we fix $C=10$ with a tolerance of $.1$ for the duality gap.  For the projected gradient method, increasing $\rho$ increases the number of iterations to converge; notice that the average time per iteration does not vary over $\rho$.  SIQCLP also requires more iterations to converge for higher $\rho$, however the average iteration time seems to be less for higher $\rho$, so no clear pattern is seen when varying $\rho$.  Note that the number of iterations required varies widely (between 100 and 2000 iterations in this experiment) as a function of $\rho$, $C$, the chosen kernel and the stepsize.

Results for ACCPM and SIQCLP are shown only up to dimensions 500 and 300, respectively, because this sufficiently demonstrates that the projected gradient method is more efficient.  ACCPM clearly suffers from the complexity of the analytic center problem each iteration.  However, improvements can be made in the SIQCLP implementation such as using a regularized version of an efficient MKL solver (e.g. \citeasnoun{Rako2007}) to solve problem (\ref{eq:SIQCLP_master}) rather than MOSEK.  SIQCLP is also useful because it makes a connection between the indefinite SVM formulation and multiple kernel learning.  We observed from experiments that the duality gap found from SIQCLP is tighter than the upper bound on the duality gap used for the projected gradient method.  This could potentially be used to create a better stopping condition, however the complexity to derive the tighter duality gap (solving regularized MKL) is much higher than that to compute our current gap (solving a single SVM).

\begin{figure} [h!]
\begin{tabular*}{\textwidth}{@{\extracolsep{\fill}}cc}
\psfrag{time}[b][t]{\small{Total Time (sec)}}
\psfrag{dim}[t][b]{\small{Dimension}}
\includegraphics[width=.50\textwidth]{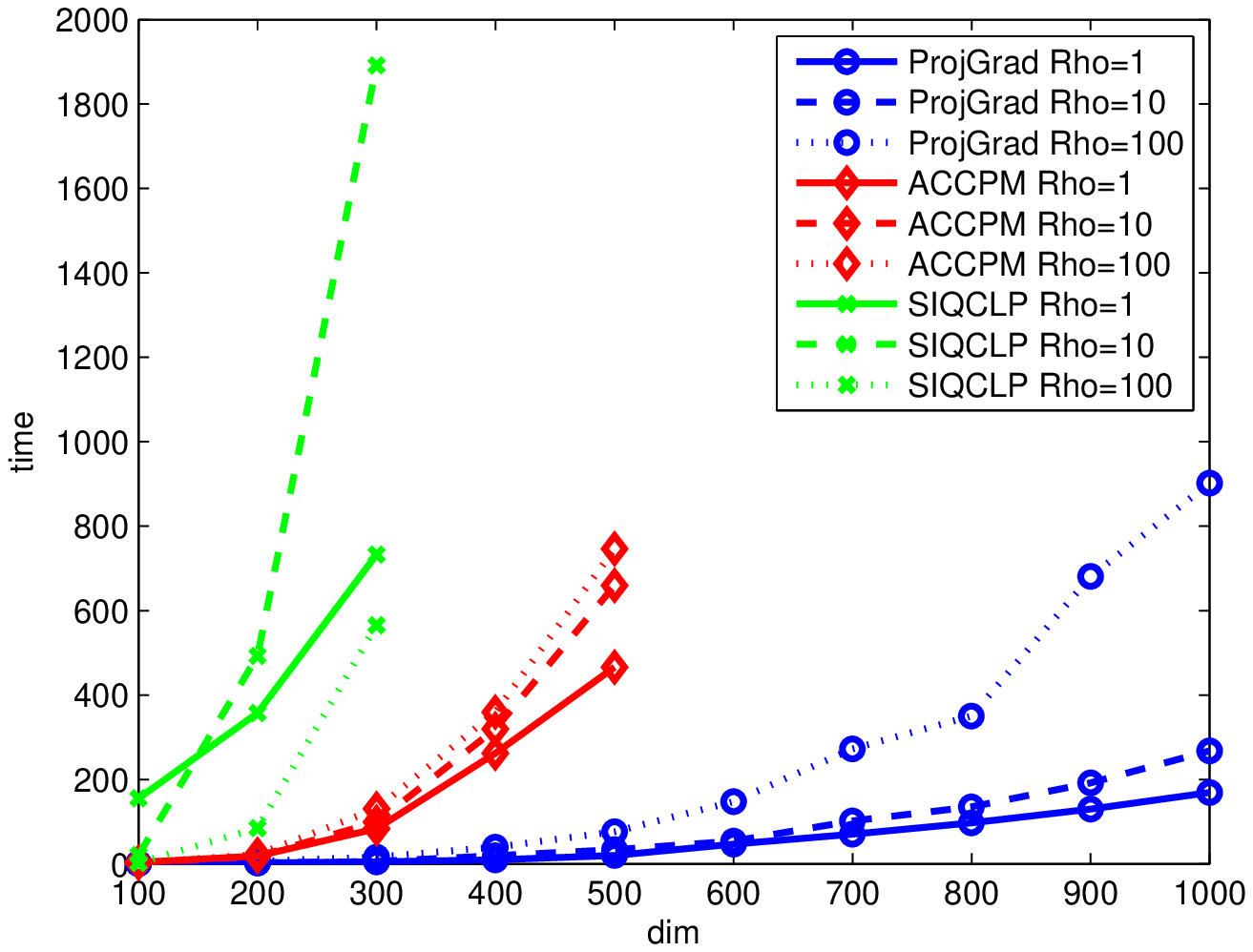}
&
\psfrag{time}[b][t]{\small{Average Time/Iteration (sec)}}
\psfrag{dim}[t][b]{\small{Dimension}}
\includegraphics[width=.50 \textwidth]{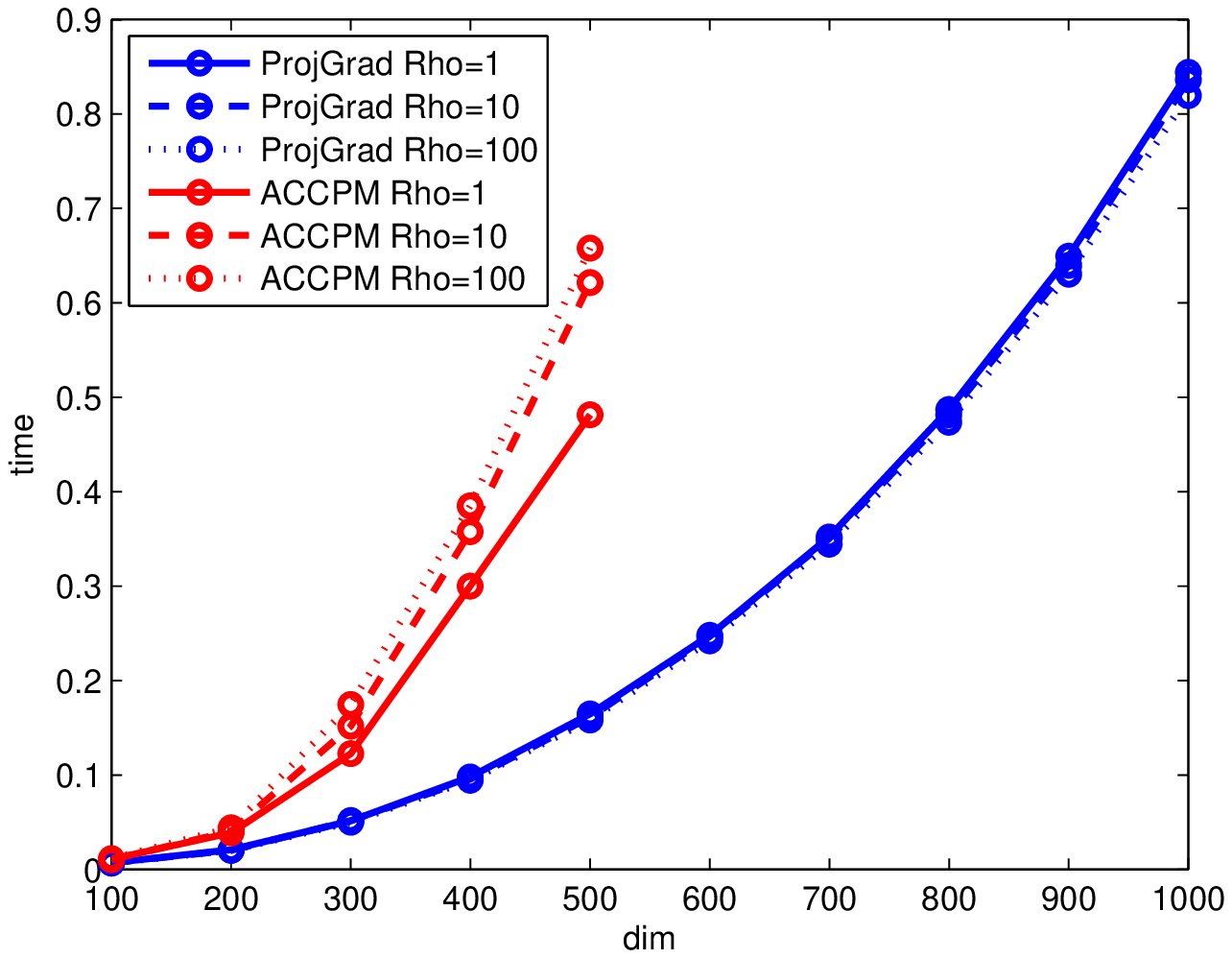} \end{tabular*}
\caption{Total time versus dimension (left) and average time per iteration versus dimension (right) using projected gradient and ACCPM IndefiniteSVM and SIQCLP (only for total time).  The number of iterations for convergence varies from 100 for the smallest dimension to 2000 for the largest dimension in this example which uses a Simpson kernel on the USPS 3-5 data.}
\label{fig:time_figures}
\end{figure}

\section{Conclusion}
We have proposed a technique for support vector machine classification with indefinite kernels, using a proxy kernel which can be computed explicitly. We also show how this framework can be used to improve generalization performance with potentially noisy Mercer kernels, as well as extend it to other kernel methods such as support vector regression and one-class support vector machines. We give two provably convergent algorithms for solving this problem on relatively large data sets. Our initial experiments show that our method fares quite favorably compared to other techniques handling indefinite kernels in the SVM framework and, in the limit, provides a clear interpretation for some of these heuristics.

\section*{Acknowledgements}
We are very grateful to M\'aty\'as Sustik for his rank-one update eigenvalue decomposition code and to Jianhui Chen and Jieping Ye for their SIQCLP Matlab code. We would also like to acknowledge support from NSF grant DMS-0625352, NSF CDI grant SES-0835550, a NSF CAREER award, a Peek junior faculty fellowship and a Howard B. Wentz Jr. junior faculty award.

\small
\bibliographystyle{agsm}
\bibliography{IndKernelLearning}

\end{document}

%% file: defs.tex
\newcommand{\BEAS}{\begin{eqnarray*}}
\newcommand{\EEAS}{\end{eqnarray*}}
\newcommand{\BEA}{\begin{eqnarray}}
\newcommand{\EEA}{\end{eqnarray}}
\newcommand{\BEQ}{\begin{equation}}
\newcommand{\EEQ}{\end{equation}}
\newcommand{\BIT}{\begin{itemize}}
\newcommand{\EIT}{\end{itemize}}
\newcommand{\BNUM}{\begin{enumerate}}
\newcommand{\ENUM}{\end{enumerate}}

\newcommand{\BA}{\begin{array}}
\newcommand{\EA}{\end{array}}
\newcommand{\BC}{\begin{center}}
\newcommand{\EC}{\end{center}}
\newcommand{\DS}{\displaystyle}

\newcommand{\reals}{{\mbox{\bf R}}}

\newcommand{\symm}{{\mbox{\bf S}}}  

\newcommand{\Tr}{\mathop{\bf Tr}}
\newcommand{\diag}{\mathop{\bf diag}}

\newcommand{\QED}{~~\rule[-1pt]{6pt}{6pt}}

\newcommand{\argmin}{\mathop{\rm argmin}}

\newcommand{\argmax}{\mathop{\rm argmax}}

\newcounter{exno}

{\begin{quote}}{\end{quote}}

\makeatletter
\long\def\@makecaption#1#2{
   \vskip 9pt
   \begin{small}
   \setbox\@tempboxa\hbox{{\bf #1:} #2}
   \ifdim \wd\@tempboxa > 5.5in
        \begin{center}
        \begin{minipage}[t]{5.5in}
        \addtolength{\baselineskip}{-0.95pt}
        {\bf #1:} #2 \par
        \addtolength{\baselineskip}{0.95pt}
        \end{minipage}
        \end{center}
   \else
    \hbox to\hsize{\hfil\box\@tempboxa\hfil}
   \fi
   \end{small}\par
}
\makeatother

\newcounter{oursection}

\newcounter{lecture}

\newtheorem{theorem}{Theorem}

\newtheorem{corollary}[theorem]{Corollary}

\newenvironment{proof}{\textbf{Proof.}}{\QED\bigskip}